\documentclass{article}

\PassOptionsToPackage{numbers,sort&compress}{natbib}
\usepackage[final]{nips_2017}
\usepackage{times}
\usepackage[utf8]{inputenc} 
\usepackage[T1]{fontenc}    

\usepackage[colorlinks]{hyperref}       
\usepackage{color}
\definecolor{darkblue}{rgb}{0.0,0.0,0.55}
\hypersetup{
	colorlinks = true,
	citecolor  = darkblue,
	linkcolor  = darkblue,
	citecolor  = darkblue,
	filecolor  = darkblue,
	urlcolor   = darkblue,
}

\usepackage[utf8]{inputenc} 
\usepackage[T1]{fontenc}    
\usepackage{hyperref}       
\usepackage{url}            
\usepackage{booktabs}       
\usepackage{amsfonts}       
\usepackage{nicefrac}       
\usepackage{microtype}      
\usepackage{enumitem}
\usepackage{wrapfig}
\usepackage{sidecap}

\usepackage{amsmath,amssymb,comment,mathrsfs}
\usepackage{amsthm,cancel}

\usepackage{graphicx} 
\usepackage{subcaption} 
\usepackage{multirow}
\usepackage{xcolor}
\usepackage{mathtools}
\usepackage{relsize}
\usepackage{url}
\usepackage{nicefrac}
\usepackage{xspace}
\usepackage{xcolor}
\usepackage{todonotes}
\usepackage{algorithm}
\usepackage{algorithmic}

\def\namedlabel#1#2{\begingroup
	#2%
	\def\@currentlabel{#2}%
	\phantomsection\label{#1}\endgroup
}

\newtheorem{lemma}{Lemma}
\newtheorem{theorem}{Theorem}
\newtheorem{corollary}{Corollary}
\newtheorem{preposition}{Preposition}

\newtheorem{definition}{Definition}

\newtheorem*{lemma*}{Lemma}
\newtheorem*{theorem*}{Theorem}
\newtheorem*{corollary*}{Corollary}
\newtheorem*{remark*}{Remark}
\newtheorem*{definition*}{Definition}

\newcommand{\sgd}{\textsc{Sgd}\xspace}

\newcommand{\svrg}{\textsc{Svrg}\xspace}
\newcommand{\adam}{\textsc{Adam}\xspace}
\newcommand{\rmsprop}{\textsc{RMSprop}\xspace}

\newcommand{\gopt}{\textsc{Gradient-Focused-Optimizer}}
\newcommand{\hopt}{\textsc{Hessian-Focused-Optimizer}}
\newcommand{\gd}{\textsc{GD}\xspace}
\newcommand{\hd}{\textsc{HessianDescent}\xspace}
\newcommand{\cd}{\textsc{CubicDescent}\xspace}
\newcommand{\acd}{\textsc{ApproxCubicDescent}\xspace}

\newcommand{\reals}{\mathbb{R}}

\newcommand{\nlsum}{\sum\nolimits}

\let\emptyset\varnothing

\title{A Generic Approach for Escaping Saddle points}

\author{
  Sashank J Reddi*, Manzil Zaheer*${}^\dagger$\\
  Machine Learning Department \\
  Carnegie Mellon University \\
  Pittsburgh, PA 15213 \\
  \texttt{sjakkamr@cs.cmu.edu, manzil@cmu.edu} \\
  \And
  Suvrit  Sra \\
  Laboratory for Information \& Decision Systems \\
  Massachusetts Institute for Technology \\
  Cambridge, MA 02139 \\
  \texttt{suvrit@mit.edu} \\
  \And
  Barnab\'as  P\'oczos \\
  Machine Learning Department \\
  Carnegie Mellon University \\
  Pittsburgh, PA 15213 \\
  \texttt{bapoczos@cs.cmu.edu} \\
  \And
  Francis  Bach \\
  Departement d'Informatique \\
  Ecole Normale Superieure \\
  75012 Paris \\
  \texttt{francis.bach@inria.fr} \\
  \And
  Ruslan  Salakhutdinov \\
  Machine Learning Department \\
  Carnegie Mellon University \\
  Pittsburgh, PA 15213 \\
  \texttt{rsalakhu@cs.cmu.edu} \\
  \And
  Alexander J Smola \\
  Deep Learning \\
  Amazon Web Services \\
  Palo Alto, CA 94301\\
  \texttt{alex@smola.org} \\
}

\begin{document}

\maketitle

\begin{abstract}
A central challenge to using first-order methods for optimizing nonconvex problems is the presence of saddle points. First-order methods often get stuck at saddle points, greatly deteriorating their performance. Typically, to escape from saddles one has to use second-order methods. However, most works on second-order methods rely extensively on expensive Hessian-based computations, making them impractical in large-scale settings. To tackle this challenge, we introduce a generic framework that minimizes Hessian based computations while at the same  time provably converging to second-order critical points. Our framework carefully alternates between a first-order and a second-order subroutine, using the latter only close to saddle points, and yields convergence results competitive to the state-of-the-art. Empirical results suggest that our strategy also enjoys good practical performance.
\end{abstract}

\section{Introduction}
\label{sec:intro}
We study nonconvex \emph{finite-sum} problems of the form
\begin{equation}
\label{eq:finite-sum}
\min_{x\in \reals^d}\ f(x) := \frac{1}{n}\sum_{i=1}^n f_i(x),
\end{equation}
where neither $f:\mathbb{R}^d \rightarrow \mathbb{R}$ nor the individual functions $f_i:\mathbb{R}^d \rightarrow \mathbb{R}$ ($i \in [n]$) are necessarily convex. We operate in a general nonconvex setting except for few smoothness assumptions like Lipschitz continuity of the gradient and Hessian. 
Optimization problems of this form arise naturally in machine learning and statistics as  empirical risk minimization (ERM) and M-estimation respectively. 

In the large-scale settings, algorithms based on first-order information of functions $f_i$ are typically favored as they are relatively inexpensive and scale seamlessly. An algorithm widely used in practice is stochastic gradient descent ($\sgd$), which has the  iterative update:
\begin{equation}
x_{t+1} = x_{t} - \eta_t \nabla f_{i_t}(x_t),
\end{equation}
where $i_t \in [n]$ is a randomly chosen index and $\eta_t$ is a learning rate. Under suitable selection of the learning rate, we can show that $\sgd$ converges to a point $x$ that, in expectation, satisfies the stationarity condition $\| \nabla f(x)\| \leq \epsilon$ in $O(1/\epsilon^4)$ iterations \cite{Ghadimi13}. This result has two critical weaknesses: (i) It does not ensure convergence to local optima or second-order critical points; (ii) The rate of convergence of the $\sgd$ algorithm is slow. 

For general nonconvex problems, one has to settle for a more modest goal than  sub-optimality, as finding the global minimizer of finite-sum nonconvex problem will be in general intractably hard. Unfortunately, $\sgd$ does not even ensure second-order critical conditions such as local optimality since it can get stuck at saddle points. This issue has recently received considerable attention in the ML community, especially in the context of deep learning~\cite{Dauphin14,Dauphin15,Choromanska14}. These works argue that saddle points are highly prevalent in most optimization paths, and are the primary obstacle for training large deep networks. To tackle this issue and achieve a second-order critical point for which $\|\nabla f\| \leq \epsilon$ and $\nabla^2 f \succeq -\sqrt{\epsilon} \mathbb{I}$, we need algorithms that either use the Hessian explicitly or exploit its structure. 

A key work that explicitly uses Hessians to obtain faster convergence rates is the cubic regularization (CR) method \cite{nesterov2006}. In particular, \citet{nesterov2006} showed that CR requires $O(1/\epsilon^{3/2})$ iterations to achieve the second-order critical conditions. However, each iteration of CR is expensive as it requires computing the Hessian and solving multiple linear systems, each of which has complexity $O(d^\omega)$ ($\omega$ is the matrix multiplication constant), thus, undermining the benefit of its faster convergence. Recently, \citet{Agarwal16} designed an algorithm to solve the CR more efficiently, however, it still exhibits slower convergence in practice compared to first-order methods. Both of these approaches use Hessian based optimization in each iteration, which make them slow in practice. 

A second line of work focuses on using Hessian information (or its structure) whenever the method gets stuck at stationary points that are not second-order critical. To our knowledge, the first work in this line is~\cite{Ge15}, which shows that for a class of functions that satisfy a special property called ``strict-saddle'' property, a noisy variant of $\sgd$ can converge to a point close to a local minimum. For this class of functions, points close to saddle points have a Hessian with a large negative eigenvalue, which proves instrumental in escaping saddle points using an isotropic noise. While such a noise-based method is appealing as it only uses first-order information, it has a very bad dependence on the dimension $d$, and furthermore, the result only holds when the strict-saddle property is satisfied \cite{Ge15}. More recently, \citet{Yair16} presented a new faster algorithm that alternates between first-order and second-order subroutines. However, their algorithm is designed for the simple case of $n = 1$ in~\eqref{eq:finite-sum} and hence, can be expensive in practice.

\sidecaptionvpos{figure}{t}
\begin{SCfigure}[50]
    \hspace{-5mm}
	\includegraphics[trim={5mm 5mm 5mm 5mm},clip,  width=0.28\textwidth]{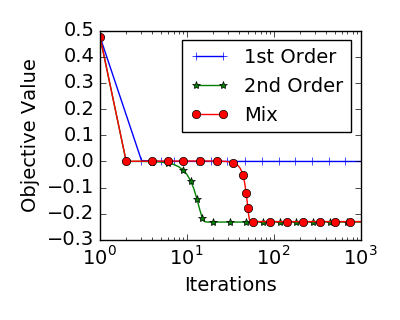}
	\includegraphics[trim={5mm 5mm 5mm 5mm},clip, width=0.28\textwidth]{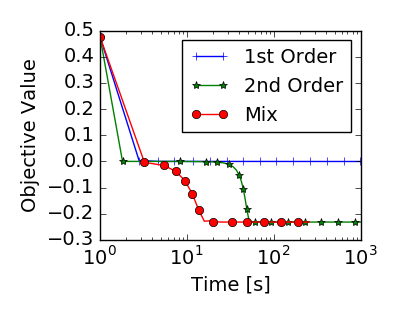}
	\caption{\small 
	First order methods like GD can potentially get stuck at saddle points. Second-order methods can escape it in very few iterations (as observed in the left plot) but at the cost of expensive Hessian based iterations (see time plot to the right). The proposed framework, which is a novel mix of the two strategies, can escape saddle points \textit{faster} in time by carefully trading off computation and iteration complexity.}
	\vspace{-5mm}
\end{SCfigure}

Inspired by this line of work, we develop a general framework for finding second-order critical points. The key idea of our framework is to use first-order information for the most part of the optimization process and invoke Hessian information only when stuck at stationary points that are not second-order critical. We summarize the key idea and main contributions of this paper below.
 
\textbf{Main Contributions:}  We develop an algorithmic framework for converging to second-order critical points and provide convergence analysis for it. Our framework  carefully alternates between two subroutines that use gradient and Hessian information, respectively, and ensures second-order criticality. Furthermore, we present two instantiations of our framework and provide convergence rates for them. In particular, we show that a simple instance of our framework, based on $\svrg$, achieves convergence rates competitive with the current state-of-the-art methods; thus highlighting the simplicity and applicability of our framework. Finally, we demonstrate the empirical performance of a few algorithms encapsulated by our framework and show their superior performance.

\paragraph{Related Work.}
\label{sec:relatedwork}
There is a vast literature on algorithms for solving optimization problems of the form~\eqref{eq:finite-sum}.
A classical approach for solving such optimization problems is $\sgd$, which dates back at least to the seminal work of \cite{RobMon51}. Since then, $\sgd$ has been a subject of extensive research, especially in the convex setting~\cite{poljak1973,ljung1977,bot91,kushner2012}. Recently, new faster methods, called variance reduced (VR) methods, have been proposed for convex finite-sum problems. VR methods attain faster convergence by reducing the variance in the stochastic updates of $\sgd$, see e.g.,~\cite{Defazio14,Johnson13,Schmidt13,Konecny15,sdca,defazio2014finito}. Accelerated variants of these methods achieve the lower bounds proved in \cite{agarwal2014,lan2015}, thereby settling the question of their optimality. Furthermore, \cite{Reddi2015} developed an asynchronous framework for VR methods and demonstrated their benefits in parallel environments.

Most of the aforementioned prior works study stochastic methods in convex or very specialized nonconvex settings that admit theoretical guarantees on sub-optimality. For the general nonconvex setting, it is only recently that non-asymptotic convergence rate analysis for \sgd and its variants was obtained in~\cite{Ghadimi13}, who showed that \sgd ensures $\|\nabla f\| \le \epsilon$ (in expectation) in $O(1/\epsilon^4)$ iterations. A similar rate for parallel and distributed \sgd was shown  in~\cite{Lian2015}. For these problems, \citet{Reddi16,Reddi16b,Reddi16c} proved faster convergence rates that ensure the same optimality criteria in $O(n + n^{2/3}/\epsilon^2)$, which is an order $n^{1/3}$ faster than $\gd$. While these methods ensure convergence to \emph{stationary} points at a faster rate, the question of convergence to local minima (or in general to second-order critical points) has not been addressed. To our knowledge, convergence rates to second-order critical points (defined in Definition~\ref{def:eps-accurate}) for general nonconvex functions was first studied by \cite{nesterov2006}. However, each iteration of the algorithm in \cite{nesterov2006} is prohibitively expensive since it requires eigenvalue decompositions, and hence, is unsuitable for large-scale high-dimensional problems. More recently, \citet{Yair16,Agarwal16} presented algorithms for finding second-order critical points by tackling some practical issues that arise in \cite{nesterov2006}. However, these algorithms are either only applicable to a restricted setting or heavily use Hessian based computations, making them unappealing from a practical standpoint. \emph{Noisy} variants of first-order methods have also been shown to escape saddle points (see \cite{Ge15,Jin17,Levy16}), however, these methods have strong dependence on either $n$ or $d$, both of which are undesirable.

\vspace*{-6pt}
\section{Background \& Problem Setup}
\vspace*{-6pt}
\label{sec:background}
We assume that each of the functions $f_i$ in~\eqref{eq:finite-sum} is $L$-smooth, i.e., $\|\nabla f_i(x)-\nabla f_i(y)\| \le L\|x-y\|$ for all $i \in [n]$. Furthermore, we  assume that the Hessian of $f$ in~\eqref{eq:finite-sum} is Lipschitz, i.e., we have
\begin{equation}
\|\nabla^2 f(x) - \nabla^2 f(y)\| \leq M \|x - y\|,
\end{equation}
for all $x, y \in \reals^d$. Such a condition is typically necessary to ensure convergence of algorithms to the second-order critical points \cite{nesterov2006}. In addition to the above smoothness conditions, we also assume that the function $f$ is bounded below, i.e., $f(x) \geq B$ for all $x \in \reals^d$. 

In order to measure stationarity of an iterate $x$, similar to \cite{nesterov03,Ghadimi13,nesterov2006}, we use the condition $\|\nabla f(x)\| \le \epsilon$. In this paper, we are interested in convergence to second-order critical points. Thus, in addition to stationarity, we also require the solution to satisfy the Hessian condition $\nabla^2 f(x) \succeq -\gamma \mathbb{I}$ \cite{nesterov2006}. For iterative algorithms, we require both $\epsilon, \gamma \rightarrow 0$ as the number of iterations $T \rightarrow \infty$. When all saddle points are non-degenerate, such a condition implies convergence to a local optimum.
\begin{definition}
	\label{def:eps-accurate}
An algorithm $\mathcal{A}$ is said to obtain a point $x$ that is a $(\epsilon,\gamma)$-second order critical point if $\mathbb{E}[\|\nabla f(x)\|] \leq \epsilon$ and $\nabla^2 f(x) \succeq -\gamma \mathbb{I}$, where the expectation is over any randomness in $\mathcal{A}$.
\vspace{-1mm}
\end{definition}

We must exercise caution while interpreting results pertaining to $(\epsilon,\gamma)$-second order critical points. Such points need not be close to any local minima either in objective function value, or in the domain of~\eqref{eq:finite-sum}. For our algorithms, we use only an Incremental First-order Oracle (IFO)~\cite{agarwal2014} and an Incremental Second-order Oracle (ISO), defined below.

\begin{definition}
  An IFO takes an index $i \in [n]$ and a point $x \in \mathbb{R}^d$, and returns the pair $(f_i(x),\nabla f_i(x))$. An ISO takes an index $i \in [n]$, point $x \in \mathbb{R}^d$ and vector $v \in \mathbb{R}^d$ and returns the vector $\nabla^2 f_i(x) v$.
  \vspace{-1mm}
 \end{definition}
IFO and ISO calls are typically cheap, with ISO call being relatively more expensive. In many practical settings that arise in machine learning, the time complexity of these oracle calls is linear in $d$ \cite{Agarwal16b,Pearlmutter94}. For clarity and clean comparison, the dependence of time complexity on Lipschitz constant $L$, $M$, initial point and any polylog factors in the results is hidden.

\vspace*{-6pt}
\section{Generic Framework}
\label{sec:framework}
\vspace*{-6pt}
In this section, we propose a generic framework for escaping saddle points while solving nonconvex problems of form~\eqref{eq:finite-sum}. One of the primary difficulties in reaching a second-order critical point is the presence of saddle points. To evade such points, one needs to use properties of both gradients and Hessians. To this end, our framework is based on two core subroutines: $\gopt$ and $\hopt$. 

The idea is to use these two subroutines, each focused on different aspects of the optimization procedure. $\gopt$ focuses on using gradient information for decreasing the function. On its own, the $\gopt$ might not converge to a local minimizer since it can get stuck at a saddle point. Hence, we require the subroutine $\hopt$ to help avoid saddle points. A natural idea is to interleave these subroutines to obtain a second-order critical point. But it is not even clear if such a procedure even converges. We propose a carefully designed procedure that effectively balances these two subroutines, which not only provides meaningful theoretical guarantees, but remarkably also translates into strong empirical gains in practice.

Algorithm~\ref{alg:framework} provides pseudocode of our framework. Observe that the algorithm is still abstract, since it does not specify the subroutines $\gopt$ and $\hopt$. These subroutines determine the crucial update mechanism of the algorithm. We will present specific instance of these subroutines in the next section, but we assume the following properties to hold for these subroutines.

\begin{algorithm}[tb]\small
	\caption{Generic Framework}
	\label{alg:framework}
	\begin{algorithmic}[1]
		\STATE{\bf Input} - Initial point: $x^0$, total iterations $T$, error threshold parameters $\epsilon$, $\gamma$ and probability $p$
		\FOR{$t=1$ {\bfseries to} $T$}
		\STATE $(y^{t}, z^t) = \gopt(x^{t-1}, \epsilon)$ (refer to \ref{itm:ga1} and \ref{itm:ga2})
		\STATE Choose $u^t$ as $y^t$ with probability $p$ and $z^t$ with probability $1 - p$
		\STATE $(x^{t+1},\tau^{t+1}) = \hopt(u^t, \epsilon, \gamma)$ (refer to \ref{itm:ha1} and \ref{itm:ha2})
		\IF{$\tau^{t+1} = \emptyset$}
		\STATE {\bf Output} set $\{x^{t+1}\}$
		\ENDIF
		\ENDFOR
		\STATE{\bf Output} set $\{y^1, ... , y^T\}$
	\end{algorithmic}
\end{algorithm}

\vspace{-2mm}
\begin{itemize}[leftmargin=*]
	\item $\gopt$:   Suppose $(y,z)$ = $\gopt(x, n, \epsilon)$, then there exists positive function $g: \mathbb{N} \times \mathbb{R}^+ \rightarrow \mathbb{R}^+$, such that
	\begin{enumerate}[label=\textbf{G.\arabic*}]
		\item \label{itm:ga1} $\mathbb{E}[f(y)] \leq f(x)$,
		\item \label{itm:ga2} $\mathbb{E}[\|\nabla f(y)\|^2] \leq \frac{1}{g(n,\epsilon)}  \mathbb{E}[f(x) - f(z)]$.
	\end{enumerate}
	Here the outputs $y,z \in \mathbb{R}^d$. The expectation in the conditions above is over any randomness that is a part of the subroutine. The function $g$ will be critical for the overall rate of Algorithm~\ref{alg:framework}. Typically, $\gopt$ is a first-order method, since the primary aim of this subroutine is to focus on gradient based optimization.
	\item $\hopt$: Suppose $(y,\tau) = \hopt(x, n, \epsilon, \gamma)$ where $y \in \mathbb{R}^d$ and $\tau \in \{\emptyset, \diamond\}$. If $\tau = \emptyset$, then $y$ is a $(\epsilon,\gamma)$-second order critical point with probability at least $1 - q$. Otherwise if $\tau = \diamond$, then $y$ satisfies the following condition:
          \begin{enumerate}[label=\textbf{H.\arabic*}]
          \item \label{itm:ha1} $\mathbb{E}[f(y)] \leq f(x)$,
          \item \label{itm:ha2} $\mathbb{E}[f(y)] \leq f(x) - h(n, \epsilon, \gamma)$ when $\lambda_{\min}(\nabla^2 f(x)) \leq - \gamma$ for some function $h:\mathbb{N} \times \mathbb{R}^+ \times \mathbb{R}^+ \rightarrow \mathbb{R}^+$.
          \end{enumerate}
	Here the expectation is over any randomness in subroutine $\hopt$. The two conditions ensure that the objective function value, in expectation, never increases and furthermore, decreases with a certain rate when $\lambda_{\min}(\nabla^2 f(x)) \leq - \gamma$. In general, this subroutine utilizes the Hessian or its properties for minimizing the objective function. Typically, this is the most expensive part of the Algorithm~\ref{alg:framework} and hence, needs to be invoked judiciously.
\end{itemize}
\vspace{-2mm}

The key aspect of these subroutines is that they, in expectation, never increase the objective function value. The functions $g$ and $h$ will determine the convergence rate of Algorithm~\ref{alg:framework}. In order to provide a concrete implementation, we need to specify the aforementioned subroutines. Before we delve into those details, we will provide a generic convergence analysis for Algorithm~\ref{alg:framework}.

\vspace{-2mm}
\subsection*{Convergence Analysis} 

\begin{theorem}
\label{thm:generic-conv}
Let $\Delta = f(x^0) - B$ and $\theta = \min((1-p)\epsilon^2 g(n,\epsilon), p h(n,\epsilon,\gamma))$ . Also, let set $\Gamma$ be the output of Algorithm~\ref{alg:framework} with $\gopt$ satisfying \ref{itm:ga1} and \ref{itm:ga2} and $\hopt$ satisfying \ref{itm:ha1} and \ref{itm:ha2}.  Furthermore, $T$ be such that $T > \Delta/\theta$. 

Suppose the multiset $S = \{i_1, ... i_k\}$ are $k$ indices selected independently and uniformly randomly from \{1, ..., $|\Gamma|$\}. Then the following holds for the indices in $S$:
\begin{enumerate}[leftmargin=*]
	\item $y^{t}$, where $t \in \{i_1, ..., i_k\}$, is a $(\epsilon, \gamma)$-critical point with probability at least $1 - \max(\Delta/(T\theta), q)$.
	\item If $k = O(\log(1/\zeta)/\min(\log(\Delta/(T\theta)), \log(1/q)))$, with at least probability $1 - \zeta$,  at least one iterate $y^{t}$ where $t \in \{i_1, ..., i_k\}$ is a $(\epsilon, \gamma)$-critical point.
\end{enumerate}
\end{theorem}

The proof of the result is presented in Appendix~\ref{app:proof-conv}. The key point regarding the above result is that the overall convergence rate depends on the magnitude of both functions $g$ and $h$. Theorem ~\ref{thm:generic-conv} shows that the slowest amongst the subroutines $\gopt$ and $\hopt$ governs the overall rate of Algorithm~\ref{alg:framework}. Thus, it is important to ensure that both these procedures have good convergence. Also, note that the optimal setting for $p$ based on the result above satisfies $1/p = 1/\epsilon^2g(n,\epsilon) + 1/h(n,\epsilon,\gamma)$ . We defer further discussion of convergence to next section,  where we present more specific convergence and rate analysis.

\section{Concrete Instantiations}
\label{sec:algorithm}
We now present specific instantiations of our framework in this section. 
Before we state our key results, we discuss an important subroutine that is used as $\gopt$ for rest of this paper: \svrg. We give a brief description of the algorithm in this section and show that it meets the conditions required for a $\gopt$. \svrg~\cite{Johnson13,Reddi16} is a stochastic algorithm recently shown to be very effective for reducing variance in finite-sum problems. We seek to understand its benefits for nonconvex optimization, with a particular focus on the issue of escaping saddle points. Algorithm~\ref{alg:svrg} presents \svrg's pseudocode.

\begin{algorithm}[tb]\small
	\caption{SVRG$\left(x^0, \epsilon\right)$}
	\label{alg:svrg}
	\begin{algorithmic}[1]
		\STATE {\bfseries Input:} $x^0_m = x^0 \in \mathbb{R}^d$,  epoch length $m$, step sizes $\{\eta_i > 0\}_{i=0}^{m-1}$, iterations $T_g$, $S = \lceil T_g/m \rceil$
		\FOR{$s=0$ {\bfseries to} $S-1$}
		\STATE $\tilde{x}^{s} = x^{s+1}_0 = x^{s}_m$
		\STATE $g^{s+1} = \frac{1}{n} \sum_{i=1}^n \nabla f_{i}(\tilde{x}^{s})$
		\FOR{$t=0$ {\bfseries to} $m-1$}
		\STATE Uniformly randomly pick $i_t$ from $\{1, \dots, n\}$ \label{alg1:ln:sample}
		\STATE $v_t^{s+1} =  \nabla f_{i_t}(x^{s+1}_t) - \nabla f_{i_t}(\tilde{x}^{s}) + g^{s+1}$ \label{alg1:ln:update}
		\STATE $x^{s+1}_{t+1} = x^{s+1}_{t} - \eta_t v_t^{s+1} $
		\ENDFOR
		\ENDFOR
		\STATE {\bfseries Output:} $(y,z)$ where $y$ is Iterate $x_a$ chosen uniformly random from $\{\{x^{s+1}_t\}_{t=0}^{m-1}\}_{s=0}^{S-1}$ and $z = x^{S}_m$.
	\end{algorithmic}
\end{algorithm}

Observe that Algorithm~\ref{alg:svrg} is an epoch-based algorithm. At the start of each epoch $s$, a full gradient is calculated at the point $\tilde{x}^{s}$, requiring $n$ calls to the IFO. Within its inner loop \svrg performs $m$ stochastic updates. Suppose $m$ is chosen to be $O(n)$ (typically used in practice), then the total IFO calls per epoch is $\Theta(n)$. Strong convergence rates have been proved Algorithm~\ref{alg:svrg} in the context of convex and nonconvex optimization \cite{Johnson13,Reddi16}. The following result shows that $\svrg$ meets the requirements of a $\gopt$.
\begin{lemma}
\label{lem:svrg}
Suppose $\eta_t = \eta = 1/4Ln^{2/3}$, $m = n$ and $T_g = T_\epsilon$,  which depends on $\epsilon$, then Algorithm~\ref{alg:svrg} is a $\gopt$ with $g(n,\epsilon) = T_\epsilon/40Ln^{2/3}$.
\end{lemma}
In rest of this section, we discuss approaches using \svrg as a $\gopt$. In particular, we propose and provide convergence analysis for two different methods with different $\hopt$ but which use \svrg as a $\gopt$.

\subsection{Hessian descent}
\label{sec:hd}
The first approach is based on directly using the eigenvector corresponding to the smallest eigenvalue as a $\hopt$. More specifically, when the smallest eigenvalue of the Hessian is negative and reasonably large in magnitude, the Hessian information can be used to ensure descent in the objective function value. The pseudo-code for the algorithm is given in Algorithm~\ref{alg:hdescent}.

The key idea is to utilize the minimum eigenvalue information in order to make a descent step. If  $\lambda_{\min}(\nabla^2 f(x)) \leq -\gamma$ then the idea is to use this information to take a descent step. Note the subroutine is designed in a fashion such that the objective function value never increases. Thus, it naturally satisfies the requirement \ref{itm:ha1} of $\hopt$. The following result shows that $\hd$ is a $\hopt$.

\begin{lemma}
\label{lem:hdescent-hopt}
\hd is a $\hopt$ with $h(n,\epsilon,\gamma) = \frac{\rho}{24M^2} \gamma^3$.
\end{lemma}
The proof of the result is presented in Appendix~\ref{sec:hdescent-appendix}. With \svrg as $\gopt$ and \hd as $\hopt$, we show the following key result:
\begin{theorem}
\label{thm:hessian-descent-main}
Suppose \svrg with $m = n$, $ \eta_t = \eta = 1/4Ln^{2/3}$ for all $t \in \{1,...,m\}$ and $T_g = 40Ln^{2/3}/\epsilon^{1/2}$ is used as $\gopt$ and $\hd$ is used as $\hopt$ with $q = 0$, then Algorithm~\ref{alg:framework} finds a $(\epsilon, \sqrt{\epsilon})$-second order critical point in $T = O(\Delta/\min(p,1-p)\epsilon^{3/2})$ with probability at least $0.9$.
\end{theorem}
The result directly follows from using Lemma~\ref{lem:svrg} and~\ref{lem:hdescent-hopt} in Theorem~\ref{thm:generic-conv}. The result shows that the iteration complexity of Algorithm~\ref{alg:framework} in this case is $O(\Delta/\epsilon^{3/2}\min(p,1-p))$. Thus, the overall IFO complexity of \svrg algorithm is $(n + T_g) \times T = O(n/\epsilon^{3/2} + n^{2/3}/\epsilon^2)$. Since each IFO call takes $O(d)$ time, the overall time complexity of all $\gopt$ steps is $O(nd/\epsilon^{3/2} + n^{2/3}d/\epsilon^2)$. To understand the time complexity of $\hd$, we need the following result \cite{Agarwal16}. 

\begin{algorithm}[t]\small
	\caption{\hd$\left(x, \epsilon, \gamma\right)$}
	\label{alg:hdescent}
	\begin{algorithmic}[1]
		\STATE Find $v$ such that $\|v\| = 1$, and with probability at least $\rho$ the following inequality holds: $\langle v, \nabla^2 f(x) v \rangle \leq \lambda_{min}(\nabla^2 f(x)) + \frac{\gamma}{2}$.
		\STATE Set $\alpha = |\langle v, \nabla^2 f(x) v \rangle|/M$.
		\STATE $u = x - \alpha \ \text{sign}(\langle v, \nabla f(x) \rangle) v$.
		\STATE $y = \arg\min_{z \in \{u, x\}} f(z)$
		\STATE {\bfseries Output:} $(y,\diamond)$.
	\end{algorithmic}
\end{algorithm}

\begin{preposition}
The time complexity of finding $v \in \mathbb{R}^d$ that $\|v\| = 1$, and with probability at least $\rho$ the following inequality holds: $\langle v, \nabla^2 f(x) v \rangle \leq \lambda_{min}(\nabla^2 f(x)) + \frac{\gamma}{2}$ is $O(nd + n^{3/4}d/\gamma^{1/2})$.
\end{preposition}
Note that each iteration of Algorithm~\ref{alg:framework} in this case has just linear dependence on $d$. Since the total number of $\hd$ iterations is $O(\Delta/\min(p,1-p)\epsilon^{3/2})$ and each iteration has the complexity of $O(nd + n^{3/4}d/\epsilon^{1/4})$, using the above remark, we obtain an overall time complexity of $\hd$ is $O(nd/\epsilon^{3/2} + n^{3/4}d/\epsilon^{7/4})$. Combining this with the time complexity of \svrg, we get the following result.

\begin{corollary}
\label{cor:hessian-descent}
The overall running time of Algorithm~\ref{alg:framework} to find a $(\epsilon,\sqrt{\epsilon})$-second order critical point, with parameter settings used in Theorem~\ref{thm:hessian-descent-main}, is $O(nd/\epsilon^{3/2} + n^{3/4}d/\epsilon^{7/4} + n^{2/3}d/\epsilon^2)$.
\end{corollary} 

Note that the dependence on $\epsilon$ is much better in comparison to that of Noisy SGD used in \cite{Ge15}. Furthermore, our results are competitive with \cite{Agarwal16, Yair16} in their respective settings, but with a much simpler algorithm and analysis. We also note that our algorithm is faster than the one proposed in~\cite{Jin17}, which has a time complexity of $O(nd/\epsilon^2)$. 

\subsection{Cubic Descent} 

In this section, we show that the cubic regularization method in \cite{nesterov2006} can be used as $\hopt$. More specifically,  here $\hopt$ approximately solves the following optimization problem:
\begin{equation}
\tag{$\cd$}
y = \arg\min_{z}  \left\langle \nabla f(x), z - x\right\rangle + \frac{1}{2} \left\langle z - x, \nabla^2 f(x) (z - x) \right\rangle + \frac{M}{6} \|z - x\|^3,
\end{equation}
and returns $(y,\diamond)$ as output. The following result can be proved for this approach.

\begin{theorem}
\label{thm:cubic-descent-main}
Suppose \svrg (same as Theorem~\ref{thm:hessian-descent-main}) is used as $\gopt$ and $\cd$ is used as $\hopt$ with $q = 0$, then Algorithm~\ref{alg:framework} finds a $(\epsilon, \sqrt{\epsilon})$-second order critical point in $T = O(\Delta/\min(p,1-p)\epsilon^{3/2})$ with probability at least $0.9$.
\end{theorem}

In principle, Algorithm~\ref{alg:framework} with $\cd$ as $\hopt$ can converge without the use of $\gopt$ subroutine at each iteration since it essentially reduces to the cubic regularization method of~\cite{nesterov2006}. However, in practice, we would expect $\gopt$ to perform most of the optimization and $\hopt$ to be used for far fewer  iterations. Using the method developed in~\cite{nesterov2006} for solving $\cd$, we obtain the following corollary.

\begin{corollary}
The overall running time of Algorithm~\ref{alg:framework} to find a $(\epsilon,\sqrt{\epsilon})$-second order critical point, with parameter settings used in Theorem~\ref{thm:cubic-descent-main}, is $O(nd^\omega/\epsilon^{3/2} + n^{2/3}d/\epsilon^2)$.
\end{corollary} 

Here $\omega$ is the matrix multiplication constant. The dependence on $\epsilon$ is weaker in comparison to Corollary~\ref{cor:hessian-descent}. However, each iteration of $\cd$ is expensive (as seen from the factor $d^\omega$ in the corollary above) and thus, in high dimensional settings typically encountered in machine learning, this approach can be expensive in comparison to $\hd$.

\subsection{Practical Considerations}
\label{sec:practical}
The focus of this section was to demonstrate the wide applicability of our framework; wherein using a simple instantiation of this framework, we could achieve algorithms with fast convergence rates. To further achieve good empirical performance, we had to slightly modify these procedures. For $\hopt$, we found stochastic, adaptive and inexact approaches for solving $\hd$ and $\cd$ work well in practice. Due to lack of space, the exact description of these modifications is deferred to Appendix~\ref{sec:approx-cubic}. Furthermore, in the context of deep learning, empirical evidence suggests that first-order methods like $\adam$ \cite{Kingma14} exhibit behavior that is in congruence with properties \ref{itm:ga1} and \ref{itm:ga2}. While theoretical analysis for a setting where \adam is used as $\gopt$ is still unresolved, we nevertheless demonstrate its performance through empirical results in the following section.

\section{Experiments}
\label{sec:experiments}
We now present empirical results for our saddle point avoidance technique with an aim to highlight three aspects: (i) the framework successfully escapes non-degenerate saddle points, (ii) the framework is fast, and (iii) the framework is practical on large-scale problems.  All the algorithms are implemented on TensorFlow \cite{tensorflow2015-whitepaper}. In case of deep networks, the Hessian-vector product is evaluated using the trick presented in \cite{Pearlmutter94}. We run our experiments on a commodity machine with Intel$^\text{\textregistered}$ Xeon$^\text{\textregistered}$ CPU E5-2630 v4 CPU, 256GB RAM, and NVidia$^\text{\textregistered}$ Titan X (Pascal) GPU.

\begin{figure*}[t]
	\centering
	\begin{subfigure}[b]{0.28\textwidth}
		\includegraphics[trim={5mmm 5mm 5mm 5mm},clip, height=0.8\textwidth, width=\textwidth]{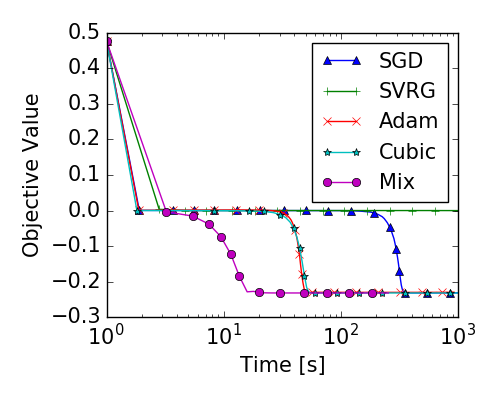}
	\end{subfigure}
	\begin{subfigure}[b]{0.28\textwidth}
		\includegraphics[trim={5mm 5mm 5mm 5mm},clip, height=0.8\textwidth, width=\textwidth]{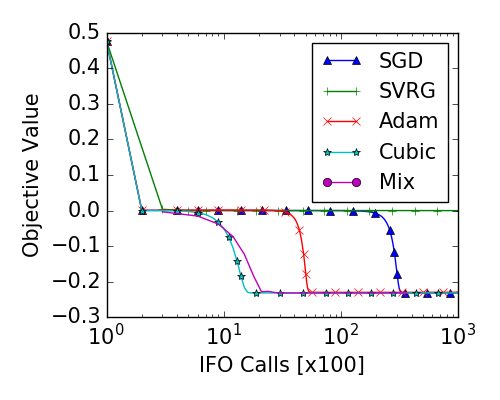}
	\end{subfigure}
	\begin{subfigure}[b]{0.28\textwidth}
		\includegraphics[trim={5mm 5mm 5mm 5mm},clip, height=0.8\textwidth, width=\textwidth]{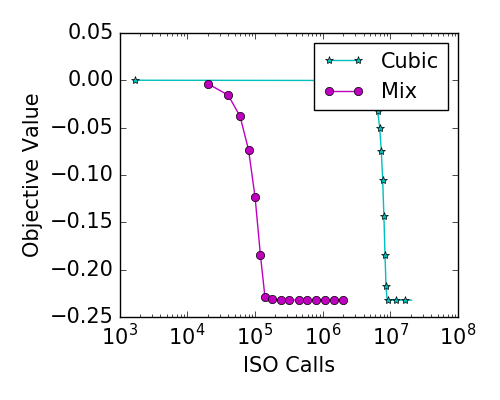}
	\end{subfigure}
	\caption{Comparison of various methods on a synthetic problem. Our mix framework successfully escapes saddle point and uses relatively few ISO calls in comparison to $\cd$.}
	\label{fig:toy-results}
    \vspace{-2mm}
\end{figure*}

{\bf Synthetic Problem} To demonstrate the fast escape from a saddle point by the proposed method, we consider the following simple nonconvex finite-sum problem:
\begin{equation}
\min_{x\in\reals^d} \frac{1}{n}\sum_{i=1}^n x^TA_ix + b_i^Tx + \|x\|_{10}^{10}
\end{equation}
Here the parameters are designed such that $\sum_i b_i = 0$ and $\sum_i A_i$ matrix has exactly one negative eigenvalue of $-0.001$ and other eigenvalues randomly chosen in the interval $[1,2]$. The total number of examples $n$ is set to be 100,000 and $d$ is $1000$. It is not hard to see that this problem has a non-degenerate saddle point at the origin. This allows us to explore the behaviour of different optimization algorithms in the vicinity of the saddle point. In this experiment, we compare a mix of \svrg and \hd (as in Theorem~\ref{thm:hessian-descent-main}) with $\sgd$ (with constant step size), $\adam$, $\svrg$ and $\cd$. The parameter of these algorithms is chosen by grid search so that it gives the best performance. The subproblem of $\cd$ was solved with gradient descent \cite{Yair16} until the gradient norm of the subproblem is reduced below $10^{-3}$. We study the progress of optimization, i.e., decrease in function value with wall clock time, IFO calls, and ISO calls. All algorithms were initialized with the same starting point very close to origin.

The results are presented in Figure~\ref{fig:toy-results}, which shows that our proposed mix framework was the \textit{fastest} to escape the saddle point in terms of wall clock time. We observe that performance of the first order methods suffered severely due to the saddle point. Note that \sgd eventually escaped the saddle point due to inherent noise in the mini-batch gradient. \cd, a second-order method, escaped the saddle point faster in terms of iterations using the Hessian information. But operating on Hessian information is expensive as a result this method was slow in terms of wall clock time. The proposed framework, which is a mix of the two strategies, inherits the best of both worlds by using cheap gradient information most of the time and reducing the use of relatively expensive Hessian information (ISO calls) by 100x. This resulted in \textit{faster} escape from saddle point in terms of wall clock time.

{\bf Deep Networks} To investigate the practical performance of the framework for deep learning problems, we applied it to two deep autoencoder optimization problems from \cite{hinton2006reducing} called “CURVES” and “MNIST”. Due to their high difficulty, performance on these problems has become a standard benchmark for neural network optimization methods, e.g. \cite{martens2015optimizing, sutskever2013importance, vinyals2012krylov, martens2010deep}. The  “CURVES” autoencoder consists of an encoder with layers of size (28x28)-400-200-100- 50-25-6 and a symmetric decoder totaling in 0.85M parameters. The six units in the code layer were linear and all the other units were logistic. The network was trained on 20,000 images and tested on 10,000 new images. The data set contains images of curves that were generated from three randomly chosen points in two dimensions. The “MNIST” autoencoder consists of an encoder with layers of size (28x28)-1000-500-250-30 and a symmetric decoder, totaling in 2.8M parameters. The thirty units in the code layer were linear and all the other units were logistic. The network was trained on 60,000 images and tested on 10,000 new images. The data set contains images of handwritten digits 0-9. The pixel intensities were normalized to lie between 0 and 1.\footnote{Data available at: \url{www.cs.toronto.edu/~jmartens/digs3pts_1.mat}, \href{www.cs.toronto.edu/~jmartens/mnist_all.mat}{\texttt{mnist\_all.mat}}}

\begin{figure*}[t]
	\centering
	\begin{subfigure}[b]{0.245\textwidth}
		\includegraphics[trim={5mm 5mm 5mm 5mm},clip, height=0.8\textwidth, width=\textwidth]{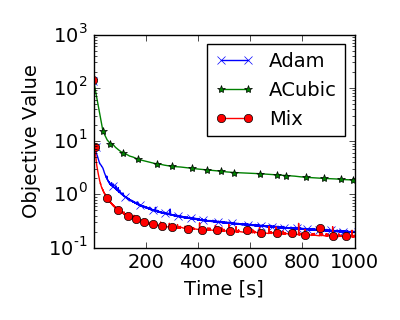}
	\end{subfigure}
	\begin{subfigure}[b]{0.245\textwidth}
		\includegraphics[trim={5mm 5mm 5mm 5mm},clip, height=0.8\textwidth, width=\textwidth]{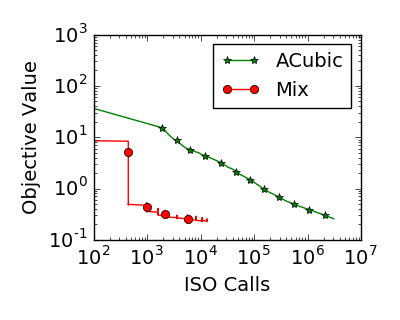}
	\end{subfigure}
	\begin{subfigure}[b]{0.245\textwidth}
		\includegraphics[trim={5mm 5mm 5mm 5mm},clip, height=0.8\textwidth, width=\textwidth]{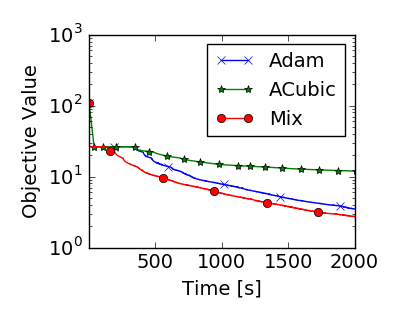}
	\end{subfigure}
	\begin{subfigure}[b]{0.245\textwidth}
		\includegraphics[trim={5mm 5mm 5mm 5mm},clip, height=0.8\textwidth, width=\textwidth]{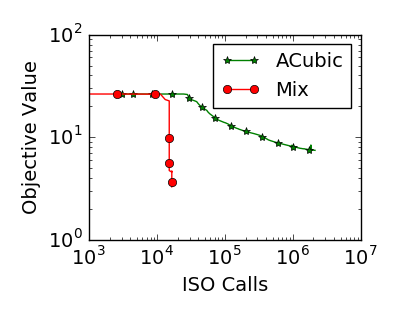}
	\end{subfigure}
	\caption{Comparison of various methods on  CURVES and MNIST Deep Autoencoder. Our mix approach converges faster than the baseline methods and uses relatively few ISO calls in comparison to $\acd$.}
	\label{fig:curves-mnist}
	\vspace{-2mm}
\end{figure*}

As an instantiation of our framework, we use a mix of \adam, which is popular in deep learning community, and an \acd for the practical reasons mentioned in Section~\ref{sec:practical}. This method with $\adam$ and $\acd$. The parameters of these algorithms were chosen to produce the best generalization on a held out test set. The regularization parameter $M$ was chosen as the smallest value such that the function value does not fluctuate in the first 10 epochs. We use the initialization suggested in \cite{martens2010deep} and a mini-batch size of 1000 for all the algorithms. We report objective function value against wall clock time and ISO calls. 

The results are presented in Figure~\ref{fig:curves-mnist}, which shows that our proposed mix framework was the \textit{fastest} to escape the saddle point in terms of wall clock time. \adam took considerably more time to escape the saddle point, especially in the case of MNIST. While \acd escaped the saddle point in relatively fewer iterations, each iteration required considerably large number of ISO calls; as a result, the method was extremely slow in terms of wall clock time, despite our efforts to improve it via approximations and code optimizations. On the other hand, our proposed framework, seamlessly balances these two methods, thereby, resulting in the fast decrease of training loss.

\vspace*{-6pt}
\section{Discussion}
\label{sec:disc}
\vspace*{-6pt}
In this paper, we examined a generic strategy to escape saddle points in nonconvex finite-sum problems and presented its convergence analysis. The key intuition is to alternate between a first-order and second-order based optimizers; the latter is mainly intended to escape points that are only stationary but are not second-order critical points. We presented two different instantiations of our framework and provided their detailed convergence analysis. While both our methods explicity use the Hessian information, one can also use noisy first-order methods as $\hopt$ (see for e.g. noisy $\sgd$ in \cite{Ge15}). In such a scenario, we exploit the negative eigenvalues of the Hessian to escape saddle points by using isotropic noise, and do not explicitly use ISO. For these methods, under strict-saddle point property \cite{Ge15}, we can show convergence to local optima within our framework.

We primarily focused on obtaining second-order critical points for nonconvex finite-sums~\eqref{eq:finite-sum}. This does not necessarily imply low test error or good generalization capabilities. Thus, we should be careful when interpreting the results presented in this paper. A detailed discussion or analysis of these issues is out of scope of this paper. While a few prior works argue for convergence to local optima, the exact connection between generalization and local optima is not well understood, and is an interesting open problem. Nevertheless, we believe the techniques presented in this paper can be used alongside other optimization tools for faster and better nonconvex optimization.

\bibliographystyle{plainnat}
\bibliography{bibfile}

\clearpage
\appendix
\begin{center}
	\bfseries\large
	Appendix: A Generic Approach for Escaping Saddle points
\end{center}

\section{Proof of Theorem~\ref{thm:generic-conv}}
\label{app:proof-conv}

The case of $\tau = \emptyset$ can be handled in a straightforward manner, so let us focus on the case where $\tau = \diamond$. We split our analysis into cases, each analyzing the change in objective function value depending on second-order criticality of $y^t$.  

We start with the case where the gradient condition of second-order critical point is violated and then proceed to the case where the Hessian condition is violated.

{\bf Case I}: $\mathbb{E}[\|\nabla f(y^t)\|] \geq \epsilon$ for some $t > 0$

We first observe the following: $\mathbb{E}[\|\nabla f(y^t)\|^2] \geq (\mathbb{E}\|\nabla f(y^t)\|)^2 \geq \epsilon^2$. This follows from a straightforward application of Jensen's inequality. From this inequality, we have the following:
\begin{align}
\epsilon^2 \leq \mathbb{E}[\|\nabla f(y^t)\|^2] \leq \frac{1}{g(n,\epsilon)} \mathbb{E}[f(x^{t-1}) - f(z^t)].
\label{eq:proof1-1}
\end{align}
This follows from the fact that $y^t$ is the output of $\gopt$ subroutine, which satisfies the condition that for $(y,z)$ = $\gopt(x, n, \epsilon)$, we have
$$
\mathbb{E}[\|\nabla f(y)\|^2] \leq \frac{1}{g(n,\epsilon)} \mathbb{E}[f(x) - f(z)].
$$
From Equation~\eqref{eq:proof1-1}, we have
$$
\mathbb{E}[f(z^t)] \leq \mathbb{E}[f(x^{t-1})] - \epsilon^2 g(n,\epsilon).
$$
Furthermore, due to the property of non-increasing nature of $\gopt$, we also have $\mathbb{E}[y^t] \leq \mathbb{E}[f(x^{t-1})]$.

We now focus on the $\hopt$ subroutine. From the property of $\hopt$ that the objective function value is non-increasing, we have $\mathbb{E}[f(x^{t})] \leq \mathbb{E}[f(u^{t})]$. Therefore, combining with the above inequality, we have
\begin{align}
\mathbb{E}[f(x^{t})]  &\leq \mathbb{E}[f(u^{t})]  \nonumber \\
&= p  \mathbb{E}[f(y^{t})] + (1 - p)  \mathbb{E}[f(z^{t})] \nonumber \\
&\leq p \mathbb{E}[f(x^{t-1})] + (1 - p) (\mathbb{E}[f(x^{t-1})] - \epsilon^2 g(n,\epsilon)) \nonumber \\
&=\mathbb{E}[f(x^{t-1})] - (1 - p) \epsilon^2 g(n,\epsilon).
\label{eq:proof1- case1}
\end{align}
The first equality is due to the definition of $u^t$ in Algorithm~\ref{alg:framework}. Therefore, when the gradient condition is violated, irrespective of whether $\lambda_{\min}(\nabla^2 f(x)) \leq - \gamma$ or $\nabla^2 f(y^t) \succeq -\gamma \mathbb{I}$, the objective function value always decreases by at least $ \epsilon^2 g(n,\epsilon)$.

{\bf Case II}: $\mathbb{E}[\|\nabla f(y^t)\|] < \epsilon$ and $\lambda_{\min}(\nabla^2 f(x)) \leq - \gamma$ for some $t > 0$

In this case, we first note that for $y = \hopt(x, n, \epsilon, \gamma)$ and $\lambda_{\min}(\nabla^2 f(x)) \leq - \gamma$,  we have $\mathbb{E}[f(y)] \leq f(x) - h(n, \epsilon,\gamma)$. Observe that $x^t = \hopt(u^t, n, \epsilon, \gamma)$. Therefore, if $u^t = y^t$ and $\lambda_{\min}(\nabla^2 f(x)) \leq - \gamma$, then we have
$$
\mathbb{E}[f(x^t)| u^t=y^t] \leq f(y^t) - h(n, \epsilon, \gamma) \leq f(x^{t-1}) - h(n, \epsilon,\gamma).
$$
The second inequality is due to the non-increasing property of $\gopt$. On the other hand, if $u^t = z^t$, we have  hand, if we have $\mathbb{E}[f(x^t)| u^t=z^t] \leq f(z^t)$. This is due to the non-increasing property of $\hopt$. Combining the above two inequalities and using the law of total expectation, we get
\begin{align}
\mathbb{E}[f(x^t)] &= p \mathbb{E}[f(x^t)| u^t=y^t]  + (1 - p) \mathbb{E}[f(x^t)| u^t=z^t] \nonumber \\
&\leq p \left( \mathbb{E}[f(y^t)] - h(n, \epsilon,\gamma) \right) + (1 - p) \mathbb{E}[f(z^t)] \nonumber  \\
&\leq p\left( \mathbb{E}[f(x^{t-1})] - h(n, \epsilon,\gamma)\right) + (1 - p) \mathbb{E}[f(x^{t-1})] \nonumber \\
&=  \mathbb{E}[f(x^{t-1})] - p h(n, \epsilon,\gamma) .
\label{eq:proof1- case2}
\end{align}

The second inequality is due to he non-increasing property of $\gopt$. Therefore, when the hessian condition is violated, the objective function value always decreases by at least $p h(n, \epsilon,\gamma)$.

{\bf Case III}: $\mathbb{E}[\|\nabla f(y^t)\|] < \epsilon$ and $\nabla^2 f(y^t) \succeq -\gamma \mathbb{I}$ for some $t > 0$

This is the favorable case for the algorithm. The only condition to note is that the objective function value will be non-increasing in this case too. This is, again, due to the non-increasing properties of subroutines $\gopt$ and $\hopt$. In general, greater the occurrence of this case during the course of the algorithm, higher will the probability that the output of our algorithm satisfies the desired property.

The key observation is that Case I \& II cannot occur large number of times since each of these cases strictly decreases the objective function value.  In particular, from Equation~\eqref{eq:proof1- case1} and~\eqref{eq:proof1- case2}, it is easy to see that each occurrence of Case I \& II the following holds:
$$
\mathbb{E}[f(x^{t})]  \leq \mathbb{E}[f(x^{t-1})]  - \theta,
$$
where $\theta  = \min((1-p)\epsilon^2 g(n,\epsilon), p h(n,\epsilon,\gamma))$. Furthermore, the function $f$ is lower bounded by B, thus, Case I \& II cannot occur more than $(f(x^0) - B)/\theta$ times. Therefore, the probability of occurrence of Case III is at least $1 - (f(x^0) - B)/(T\theta)$, which completes the first part of the proof.

The second part of the proof simply follows from first part. As seen above, the probability of Case I \& II is at most $(f(x^0) - B)/T\theta$. Therefore, probability that an element of the set $S$ falls in Case III is at least $1 - ((f(x^0) - B)/T\theta)^k$, which gives us the required result for the second part. 

\section{Proof of Lemma~\ref{lem:svrg}}

\begin{proof}
The proof follows from the analysis in~\cite{Reddi16} with some additional reasoning. We need to show two properties: \ref{itm:ga1} and \ref{itm:ga2}, both of which are based on objective function value. To this end, we start with an update in the $s^{\text{th}}$ epoch. We have the following:
	\begin{align}
	\mathbb{E}[f(x^{s+1}_{t+1})] &\leq \mathbb{E}[f(x^{s+1}_{t}) + \langle \nabla f(x^{s+1}_t), x^{s+1}_{t+1} - x^{s+1}_t \rangle \nonumber + \tfrac{L}{2} \| x^{s+1}_{t+1} - x^{s+1}_t \|^2] \nonumber \\
	&\leq \mathbb{E}[f(x^{s+1}_{t}) - \eta_t \|\nabla f(x^{s+1}_{t})\|^2 + \tfrac{L\eta_t^2}{2} \|v^{s+1}_t \|^2].
	\label{eq:svrg-proof-eq1}
	\end{align}
	The first inequality is due to $L$-smoothness of the function $f$ . The second inequality simply follows from the unbiasedness of \svrg update in Algorithm~\ref{alg:svrg}. For the analysis of the algorithm, we need the following Lyapunov function:
	$$
	A^{s+1}_{t} := \mathbb{E}[f(x^{s+1}_{t}) + \mu_{t} \|x^{s+1}_{t} - \tilde{x}^s\|^2].
	$$
	This function is a combination of objective function and the distance of the current iterate from the latest snapshot $\tilde{x}_s$. Note that the term $\mu_t$ is introduced only for the analysis and is not part of the algorithm (see Algorithm~\ref{alg:svrg}). Here $\{\mu_t\}_{t=0}^m$ is chosen such the following holds:
	$$
	\mu_{t} = \mu_{t+1}(1 + \eta_t\beta_t + 2\eta_t^2L^2 ) +  \eta_t^2L^3,
	$$
	for all $t \in \{0, \cdots, m-1\}$ and $\mu_m = 0$. For bounding the Lypunov function $A$, we need the following bound on the distance of the current iterate from the latest snapshot:
	\begin{align}
	&\mathbb{E}[\|x^{s+1}_{t+1} - \tilde{x}^s\|^2] = \mathbb{E}[\|x^{s+1}_{t+1} - x^{s+1}_t + x^{s+1}_t - \tilde{x}^s\|^2] \nonumber \\
	&= \mathbb{E}[\|x^{s+1}_{t+1} - x^{s+1}_t\|^2 + \|x^{s+1}_t - \tilde{x}^s\|^2 + 2\langle x^{s+1}_{t+1} - x^{s+1}_t, x^{s+1}_t - \tilde{x}^s\rangle] \nonumber \\
	&= \mathbb{E}[\eta_t^2\|v^{s+1}_t\|^2 + \|x^{s+1}_t - \tilde{x}^s\|^2] \nonumber  - 2\eta_t \mathbb{E}[\langle \nabla f(x^{s+1}_t), x^{s+1}_t - \tilde{x}^s\rangle]\\
	&\leq \mathbb{E}[\eta_t^2\|v^{s+1}_t\|^2 + \|x^{s+1}_t - \tilde{x}^s\|^2]  + 2 \eta_t \mathbb{E}\left[\tfrac{1}{2\beta_t} \|\nabla f(x^{s+1}_t)\|^2 + \tfrac{1}{2}\beta_t \| x^{s+1}_t - \tilde{x}^s \|^2 \right].
	\label{eq:svrg-proof-eq2}
	\end{align}
	The second equality is due to the unbiasedness of the update of $\svrg$. The last inequality follows from a simple application of Cauchy-Schwarz and Young's inequality. Substituting Equation~\eqref{eq:svrg-proof-eq1} and Equation~\eqref{eq:svrg-proof-eq2} into the Lypunov function $A^{s+1}_{t+1}$, we obtain the following:
	\begin{align}
	A^{s+1}_{t+1} &\leq \mathbb{E}[f(x^{s+1}_{t}) - \eta_t \|\nabla f(x^{s+1}_{t})\|^2 + \tfrac{L\eta_t^2}{2} \|v^{s+1}_t \|^2] \nonumber \\
	&  \qquad + \mathbb{E}[\mu_{t+1}\eta_t^2\|v^{s+1}_t\|^2 + \mu_{t+1}\|x^{s+1}_t - \tilde{x}^s\|^2] \nonumber \\
	&  \qquad + 2 \mu_{t+1}\eta_t \mathbb{E}\left[\tfrac{1}{2\beta_t} \|\nabla f(x^{s+1}_t)\|^2 + \tfrac{1}{2}\beta_t \| x^{s+1}_t - \tilde{x}^s \|^2 \right] \nonumber\\
	&\leq \mathbb{E}[f(x^{s+1}_{t}) - \left(\eta_t - \tfrac{\mu_{t+1}\eta_t}{\beta_t}\right) \|\nabla f(x^{s+1}_{t})\|^2 \nonumber\\
	&  \qquad + \left(\tfrac{L\eta_t^2}{2} + \mu_{t+1}\eta_t^2 \right)\mathbb{E}[\|v^{s+1}_t\|^2] + \left(\mu_{t+1} + \mu_{t+1}\eta_t\beta_t \right) \mathbb{E}\left[\| x^{s+1}_t - \tilde{x}^s \|^2 \right].
	\label{eq:svrg-proof-eq3}
	\end{align}
	To further bound this quantity, we use Lemma~\ref{lem:nonconvex-variance-lemma} to bound $\mathbb{E}[\|v^{s+1}_{t}\|^2]$, so that upon substituting it in Equation~\eqref{eq:svrg-proof-eq3}, we see that
	\begin{align*}
	& A^{s+1}_{t+1} \leq \mathbb{E}[f(x^{s+1}_{t})]  - \left(\eta_t - \tfrac{\mu_{t+1}\eta_t}{\beta_t} - \eta_t^2L - 2\mu_{t+1}\eta_t^2\right) \mathbb{E}[\|\nabla f(x^{s+1}_{t})\|^2] \nonumber\\
	& \qquad \qquad + \left[\mu_{t+1}\bigl(1 + \eta_t\beta_t + 2\eta_t^2L^2\bigr)+\eta_t^2L^3\right]
	\mathbb{E}\left[\| x^{s+1}_t - \tilde{x}^s \|^2 \right] \nonumber \\
	& \leq A^{s+1}_{t} - \bigl(\eta_t - \tfrac{\mu_{t+1}\eta_t}{\beta_t} - \eta_t^2L - 2\mu_{t+1}\eta_t^2\bigr) \mathbb{E}[\|\nabla f(x^{s+1}_{t})\|^2].
	\end{align*}
	The second inequality follows from the definition of $\mu_{t}$ and $A_t^{s+1}$. Since $\eta_t = \eta = 1/(4Ln^{2/3})$ for $j > 0$ and $t \in \{0, \dots, j-1\}$, 
	\begin{align}
	A^{s+1}_{j} \leq A^{s+1}_0 - \upsilon_n \nlsum_{t=0}^{j-1} \mathbb{E}[\|\nabla f(x^{s+1}_{t})\|^2],
	\label{eq:R-j-lyp}
	\end{align}
	where $$\upsilon_n =  \bigl(\eta_t - \tfrac{\mu_{t+1}\eta_t}{\beta_t} - \eta_t^2L - 2\mu_{t+1}\eta_t^2\bigr).$$
	We will prove that for the given parameter setting $ \upsilon_n > 0$ (see the proof below). With $ \upsilon_n > 0$, it is easy to see that $A^{s+1}_j \leq A^{s+1}_0$. Furthermore, note that $A^{s+1}_0 = \mathbb{E}[f(x^{s+1}_0) + \mu_0\|x^{s+1}_0 - \tilde{x}^s\|^2] = \mathbb{E}[f(x^{s+1}_0)]$ since $x^{s+1}_0 = \tilde{x}^s$ (see Algorithm~\ref{alg:svrg}). Also, we have 
	$$\mathbb{E}[f(x^{s+1}_j) +  \mu_j\|x^{s+1}_j - \tilde{x}^s\|^2] \leq \mathbb{E}[f(x^{s+1}_0)]$$ 
	and thus, we obtain $\mathbb{E}[f(x^{s+1}_j)] \leq \mathbb{E}[f(x^{s+1}_0)]$
	for all $j \in \{0, ...., m\}$. Furthermore, using  simple induction and the fact that $x^{s+1}_0 = x^s_m$ for all epoch $s \in \{0, ... ,S-1\}$, it easy to see that $\mathbb{E}[f(x^{s+1}_j)] \leq f(x^{0})$. Therefore, with the definition of $y$ specified in the output of Algorithm~\ref{alg:svrg}, we see that the condition ~\ref{itm:ga1} of $\gopt$ is satisfied for $\svrg$ algorithm.  
	
	We now prove that $\upsilon_n > 0$ and also~\ref{itm:ga2} of $\gopt$ is satisifed for $\svrg$ algorithm. By using telescoping the sum with $j = m$ in Equation~\eqref{eq:R-j-lyp}, we obtain
	\begin{align*}
	\nlsum_{t=0}^{m-1} \mathbb{E}[\|\nabla f(x^{s+1}_{t})\|^2] \leq \frac{A^{s+1}_{0} - A^{s+1}_{m}}{\upsilon_n}.
	\end{align*}
	This inequality in turn implies that
	\begin{align}
	\label{eq:descent-property}
	\nlsum_{t=0}^{m-1} \mathbb{E}[\|\nabla f(x^{s+1}_{t})\|^2] \leq \frac{\mathbb{E}[f(\tilde{x}^s) - f(\tilde{x}^{s+1})]}{ \upsilon_n},
	\end{align}
	where we used that $A^{s+1}_{m} = \mathbb{E}[f(x^{s+1}_m)] = \mathbb{E}[f(\tilde{x}^{s+1})]$ (since $\mu_m = 0$), and that $A^{s+1}_{0} = \mathbb{E}[f(\tilde{x}^s)]$ (since $x^{s+1}_0 = \tilde{x}^s$). Now sum over all epochs to obtain 
	\begin{align}
	\frac{1}{T_g} \sum_{s=0}^{S-1}\sum_{t=0}^{m-1} \mathbb{E}[\|\nabla f(x^{s+1}_{t})\|^2] \leq \frac{\mathbb{E}[f(x^{0}) - f(x^S_m)]}{T_g \upsilon_n}.
	\label{eq:nonconvex-cor-eq1}
	\end{align}
	Here we used the  the fact that $\tilde{x}^0 = x^0$. To obtain a handle on $ \upsilon_n$ and complete our analysis, we will require an upper bound on $\mu_{0}$. We observe that $\mu_0 = \tfrac{L}{16n^{4/3}} \tfrac{(1 + \theta)^m - 1}{\theta}$ where $\theta = 2\eta^2L^2 + \eta\beta$. This is obtained using the relation $\mu_{t} = \mu_{t+1}(1 + \eta\beta + 2\eta^2L^2 ) +  \eta^2L^3$ and the fact that $\mu_m = 0$. Using the specified values of $\beta$ and $\eta$ we have
	\begin{align*}
	\theta = 2\eta^2L^2 + \eta\beta = \frac{1}{8n^{4/3}} + \frac{1}{4n} \leq \frac{3}{4n}.
	\end{align*}
	Using the above bound on $\theta$, we get 
	\begin{align}
	\label{eq:c0-bound}
	\mu_0 &= \frac{L}{16n^{4/3}} \frac{(1 + \theta)^m - 1}{\theta} = \frac{L ((1 + \theta)^m - 1)}{2(1 + 2n^{1/3})} \nonumber \\
	&\leq \frac{L ((1 + \frac{3}{4n})^{\lfloor \nicefrac{4n}{3} \rfloor} - 1)}{2(1 + 2n^{1/3})}  \leq n^{-1/3}(L (e - 1)/4),
	\end{align}
	wherein the second inequality follows upon noting that $(1 + \frac{1}{l})^l$ is increasing for $l>0$ and $\lim_{l \rightarrow \infty} (1 + \frac{1}{l})^l = e$ (here $e$ is the Euler's number). Now we can lower bound $ \upsilon_n$, as
	\begin{align*}
	 \upsilon_n &= \min_t \bigl(\eta - \tfrac{\mu_{t+1}\eta}{\beta} - \eta^2L - 2\mu_{t+1}\eta^2\bigr) \geq \bigl(\eta - \tfrac{\mu_{0}\eta}{\beta} - \eta^2L - 2\mu_{0}\eta^2\bigr) \geq \frac{1}{40Ln^{2/3}}.
	\end{align*}
	The first inequality holds since $\mu_t$ decreases with $t$. The second inequality holds since (a) $\mu_0/\beta$ can be upper bounded by  $(e-1)/4$ (follows from Equation~\eqref{eq:c0-bound}), (b) $\eta^2L \leq \eta/4$ and (c) $2\mu_0\eta^2 \leq (e-1)\eta/8$ (follows from Equation~\eqref{eq:c0-bound}). Substituting the above lower bound in Equation~\eqref{eq:nonconvex-cor-eq1}, we obtain the  following:
		\begin{align}
		\frac{1}{T_g} \sum_{s=0}^{S-1}\sum_{t=0}^{m-1} \mathbb{E}[\|\nabla f(x^{s+1}_{t})\|^2] \leq \frac{40Ln^{2/3}\mathbb{E}[f(x^{0}) - f(x^S_m)]}{T_g}.
		\end{align}
	From the definition of $(y,z)$ in output of Algorithm~\ref{alg:svrg} i.e., $y$ is Iterate $x_a$ chosen uniformly random from $\{\{x^{s+1}_t\}_{t=0}^{m-1}\}_{s=0}^{S-1}$ and $z = x^{S}_m$, it is clear that Algorithm~\ref{alg:svrg} satisfies the ~\ref{itm:ga2} requirement of $\gopt$ with $g(n,\epsilon) = T_\epsilon/40Ln^{2/3}$. Since both~\ref{itm:ga1} and ~\ref{itm:ga2} are satisified for Algorithm~\ref{alg:svrg}, we conclude that $\svrg$ is a $\gopt$.
\end{proof}

\section{Proof of Lemma~\ref{lem:hdescent-hopt}}
\label{sec:hdescent-appendix}

\begin{proof}
The first important observation is that the function value never increases because $y = \arg\min_{z \in \{u, x\}} f(z)$ i.e., $f(y) \leq f(x)$, thus satisfying \ref{itm:ha1} of $\hopt$. We now analyze the scenario where $\lambda_{min}(\nabla^2 f(x)) \leq -\gamma$. Consider the event where we obtain $v$ such that
$$
\langle v, \nabla^2 f(x) v \rangle \leq \lambda_{min}(\nabla^2 f(x)) + \frac{\gamma}{2}.
$$
This event (denoted by $\mathcal{E}$) happens with at least probability $\rho$. Note that, since $\lambda_{min}(\nabla^2 f(x)) \leq -\gamma$, we have $\langle v, \nabla^2 f(x) v \rangle \leq -\tfrac{\gamma}{2}$. In this case, we have the following relationship:
\begin{align}
f(y) &\leq f(x) + \langle \nabla f(x), y - x \rangle + \frac{1}{2} (y - x)^T \nabla^2 f(x) (y - x) + \frac{M}{6} \|y - x\|^3 \nonumber \\
&= f(x) - \alpha |\langle \nabla f(x), v \rangle| + \frac{\alpha^2}{2} v^T \nabla^{2} f(x) v + \frac{M\alpha^3}{6} \|v\|^3 \nonumber \\
&\leq f(x) + \frac{\alpha^2}{2} v^T \nabla^{2} f(x) v + \frac{M\alpha^3}{6} \nonumber \\
&\leq f(x) - \frac{1}{2M^2} |v^T \nabla^{2} f(x) v|^3 + \frac{1}{6M^2} |v^T \nabla^{2} f(x) v|^3 \nonumber \\
&= f(x) - \frac{1}{3M^2} |v^T \nabla^{2} f(x) v|^3 \leq f(x) - \frac{1}{24M^2} \gamma^3.
\label{eq:hd-decrease}
\end{align}
The first inequality follows from the $M$-lipschitz continuity of the Hessain $\nabla^2 f(x)$. The first equality follows from the update rule of $\hd$. The second inequality is obtained by dropping the negative term and using the fact that  $\|v\| = 1$ . The second equality is obtained by substituting $\alpha = \frac{|v^T \nabla^{2} f(x) v|}{M}$. The last inequality is due to the fact that$\langle v, \nabla^2 f(x) v \rangle \leq -\tfrac{\gamma}{2}$.
In the other scenario where
$$
\langle v, \nabla^2 f(x) v \rangle \leq \lambda_{min}(\nabla^2 f(x)) + \frac{\gamma}{2},
$$
we can at least ensure that $f(y) \leq f(x)$ since  $y = \arg\min_{z \in \{u, x\}} f(z)$. Therefore, we have
\begin{align}
\mathbb{E}[f(y)]  &= \rho \mathbb{E}[f(y)|\mathcal{E}] + (1- \rho) \mathbb{E}[f(y)|\bar{\mathcal{E}}] \nonumber \\
&\leq \rho \mathbb{E}[f(y)|\mathcal{E}]  + (1-\rho) f(x) \nonumber \\
&\leq \rho \left[f(x) - \tfrac{\rho}{24M^2} \gamma^3\right]  + (1-\rho) f(x) \nonumber \\
&= f(x) - \tfrac{\rho}{24M^2} \gamma^3.
\end{align}
The last inequality is due to Equation~\eqref{eq:hd-decrease}. Hence, $\hopt$ satisfies \ref{itm:ha2} of $\hopt$ with $h(n,\epsilon,\gamma) = \frac{\rho}{24M^2} \gamma^3$,  thus concluding the proof.
\end{proof}

\section{Proof of Theorem~\ref{thm:cubic-descent-main}}

First note that cubic method is a descent method (refer to Theorem 1 of \cite{nesterov2006}); thus, \ref{itm:ha1} is trivially satisfied. Furthermore, cubic descent is a $\hopt$ with $h(n,\epsilon,\gamma) = \frac{2\gamma^3}{81M^3} \gamma^3$. This, again, follows from Theorem 1 of \cite{nesterov2006}. The result easily follows from the aforementioned observations.

	\section{Other Lemmas}
	
The following bound on the variance of $\svrg$ is useful for our proof \cite{Reddi16}. 
\begin{lemma}\cite{Reddi16}
	\label{lem:nonconvex-variance-lemma}
	Let $v^{s+1}_t$ be computed by Algorithm~\ref{alg:svrg}. Then,
	\begin{align*}
	\mathbb{E}[\|v^{s+1}_t\|^2] \leq 2\mathbb{E}[\|\nabla f(x^{s+1}_{t})\|^2] + 2L^2\mathbb{E}[\|x^{s+1}_{t} - \tilde{x}^{s}\|^2].
	\end{align*}
\end{lemma}
\begin{proof}
	We use the definition of $v^{s+1}_t$ to get
	\begin{align*}
	&\mathbb{E}[\|v^{s+1}_t\|^2] = \mathbb{E}[\| \left(\nabla f_{i_t}(x^{s+1}_{t}) - \nabla f_{i_t}(\tilde{x}^{s})\right) + \nabla f(\tilde{x}^{s}) \|^2] \\
	&= \mathbb{E}[\| \left(\nabla f_{i_t}(x^{s+1}_{t}) - \nabla f_{i_t}(\tilde{x}^{s})\right) + \nabla f(\tilde{x}^{s}) - \nabla f(x^{s+1}_{t}) +  \nabla f(x^{s+1}_{t}) \|^2]\\
	&\leq 2\mathbb{E}[\|\nabla f(x^{s+1}_{t})\|^2] + 2 \mathbb{E}\left[\left\| \nabla f_{i_t}(x^{s+1}_{t}) - \nabla f_{i_t}(\tilde{x}^{s}) - \mathbb{E}[\nabla f_{i_t}(x^{s+1}_{t}) - \nabla f_{i_t}(\tilde{x}^{s})] \right\|^2 \right]
	\end{align*}
	The inequality follows from the simple fact that $(a + b)^2 \leq a^2 + b^2$. From the above inequality, we get the following:
	\begin{align*}
 \mathbb{E}[\|v^{s+1}_t\|^2] &\leq 2\mathbb{E}[\|\nabla f(x^{s+1}_{t})\|^2] + 2 \mathbb{E}\|\nabla f_{i_t}(x^{s+1}_{t}) - \nabla f_{i_t}(\tilde{x}^{s})\|^2 \\
	& \leq 2\mathbb{E}[\|\nabla f(x^{s+1}_{t})\|^2] + 2L^2 \mathbb{E}[\|x^{s+1}_{t} - \tilde{x}^{s}\|^2]
	\end{align*}
	The first inequality follows by noting that for a random variable $\zeta$, $\mathbb{E}[\|\zeta - \mathbb{E}[\zeta]\|^2] \leq \mathbb{E}[\|\zeta\|^2]$. The last inequality follows from $L$-smoothness of $f_{i_t}$.
\end{proof}
	
\section{Approximate Cubic Regularization}
\label{sec:approx-cubic}
Cubic regularization method of [19] is designed to operate on full batch, i.e., it does not exploit the finite-sum structure of the problem and requires the computation of the gradient and the Hessian on the entire dataset to make an update. However, such full-batch methods do not scale gracefully with the size of data and become prohibitively expensive on large datasets. To overcome this challenge, we devised an approximate cubic regularization method described below:

\begin{enumerate}
\item Pick a mini-batch $\mathcal{B}$ and obtain the gradient and the hessian based on $\mathcal{B}$, i.e.,
\begin{equation}
g = \frac{1}{|\mathcal{B}|} \sum_{i\in \mathcal{B}} \nabla f_i(x) \qquad H = \frac{1}{|\mathcal{B}|} \sum_{i\in \mathcal{B}} \nabla^2 f_i(x)
\end{equation}
\item Solve the sub-problem
\begin{equation}
v^* = \arg\min_{v}  \left\langle g, v\right\rangle + \frac{1}{2} \left\langle v, Hv \right\rangle + \frac{M}{6} \|v\|^3
\end{equation}
\item Update: $x \leftarrow x + v^*$
\end{enumerate}

We found that this mini-batch training strategy, which requires the computation
of the gradient and the Hessian on a small subset of the dataset, to work well on a few datasets (CURVES, MNIST, CIFAR10). A similar method has been analysed in \cite{cartis2017}.

Furthermore, in many deep-networks, adaptive per-parameter learning rate helps immensely \cite{Kingma14}. One possible explanation for this is that the scale of the gradients in each layer of the network often differ by several orders of magnitude. A well-suited optimization method should take this into account. This is the reason for popularity of methods like \adam or \rmsprop in the deep learning community. On similar lines, to account for different per-parameter behaviour in cubic regularization, we modify the sub-problem by adding a diagonal matrix $M_d$ in addition to the scalar regularization coefficient $M$, i.e.,  
\begin{equation}
\min_{v}  \left\langle g, v\right\rangle + \frac{1}{2} \left\langle v, Hv \right\rangle + \frac{1}{6}  M\|M_dv\|^3.
\end{equation}
Also we devised an adaptive rule to obtain the diagonal matrix as $M_d = \mathsf{diag}((s + 10^{-12})^{1/9})$, where $s$ is maintained as a moving average of third order polynomial of the mini-batch gradient $g$, in a fashion similar to \rmsprop and \adam:
\begin{equation}
s \leftarrow \beta s + (1-\beta)(|g|^3 + 2g^2),
\end{equation}
where $|g|^3$ and $g^2$ are vectors such that $[|g|^3]_i = |g_i|^3$ and $[g^2]_i = g_i^2$  respectively for all $i \in [n]$. The experiments reported on CURVES and MNIST in this paper utilizes both the above modifications to the cubic regularization, with $\beta$ set to 0.9. We refer to this modified procedure as ACubic in our results.

\section{Experiment Details}
In this section we provide further experimental details and results to aid reproducibility.

\begin{figure*}[t]
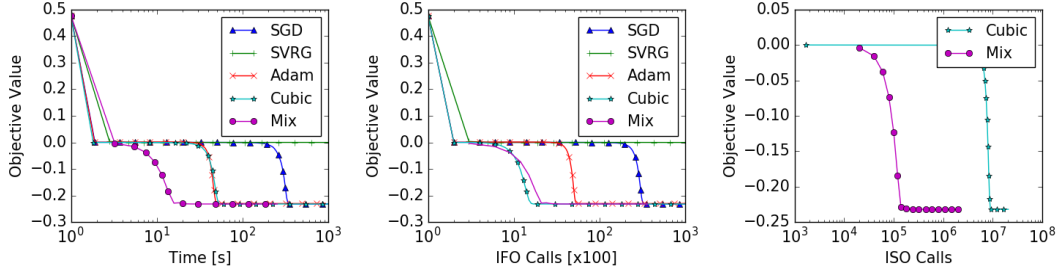

	\centering
	\begin{subfigure}[b]{0.32\textwidth}
		\includegraphics[trim={5mmm 5mm 5mm 5mm},clip, height=0.8\textwidth, width=\textwidth]{figures/toy_time.png}
	\end{subfigure}
	\hfill
	\begin{subfigure}[b]{0.32\textwidth}
		\includegraphics[trim={5mm 5mm 5mm 5mm},clip, height=0.8\textwidth, width=\textwidth]{figures/toy_dp.png}
	\end{subfigure}
	\hfill
	\begin{subfigure}[b]{0.32\textwidth}
		\includegraphics[trim={5mm 5mm 5mm 5mm},clip, height=0.8\textwidth, width=\textwidth]{figures/toy_iso.png}
	\end{subfigure}
	\caption{Comparison of various methods on a synthetic problem. Our mix framework successfully escapes saddle point. }
	\label{fig:toy-results-app}
\end{figure*}

\subsection{Synthetic Problem}
The parameter selection for all the methods were carried as follows:
\begin{enumerate}[leftmargin=*, itemsep=0mm,partopsep=0pt,parsep=0pt]
\item \sgd: The scalar step-size was determined by a grid search.
\item \adam: We performed a grid search over $\alpha$ and $\varepsilon$ parameters of \adam tied together, i.e., $\alpha=\varepsilon$.
\item \svrg: The scalar step-size was determined by a grid search.
\item \cd: The regularization parameter $M$ was chosen by grid search. The sub-problem was solved with gradient descent \cite{Yair16} with the step-size of solver to be $10^{-2}$ and run till the gradient norm of the sub-problem is reduced below $10^{-3}$.
\end{enumerate}
\vspace{-2mm}

{\bf Further Observations} The results are presented in Figure~\ref{fig:toy-results-app}. The other first order methods like \adam with higher noise could escape relatively faster whereas \svrg with reduced noise stayed stuck at the saddle point.

\subsection{Deep Networks}

\begin{figure*}[h]
	\centering
	\begin{subfigure}[b]{0.32\textwidth}
		\includegraphics[trim={5mm 5mm 5mm 5mm},clip, height=0.8\textwidth, width=\textwidth]{figures/curves_time.png}
	\end{subfigure}
	\hfill
	\begin{subfigure}[b]{0.32\textwidth}
		\includegraphics[trim={5mm 5mm 5mm 5mm},clip, height=0.8\textwidth, width=\textwidth]{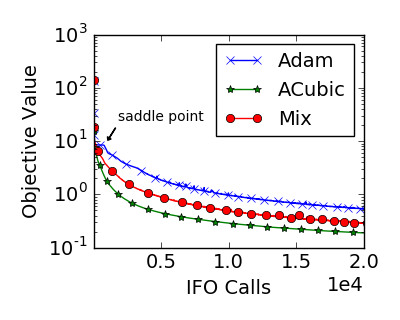}
	\end{subfigure}
	\hfill
	\begin{subfigure}[b]{0.32\textwidth}
		\includegraphics[trim={5mm 5mm 5mm 5mm},clip, height=0.8\textwidth, width=\textwidth]{figures/curves_iso.png}
	\end{subfigure}
		\begin{subfigure}[b]{0.32\textwidth}
			\includegraphics[trim={5mm 5mm 5mm 5mm},clip, height=0.8\textwidth, width=\textwidth]{figures/mnist_time.png}
		\end{subfigure}
		\hfill
		\begin{subfigure}[b]{0.32\textwidth}
			\includegraphics[trim={5mm 5mm 5mm 5mm},clip, height=0.8\textwidth, width=\textwidth]{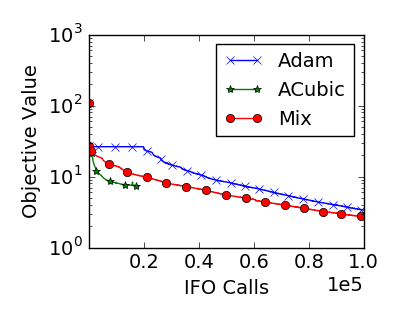}
		\end{subfigure}
		\hfill
		\begin{subfigure}[b]{0.32\textwidth}
			\includegraphics[trim={5mm 5mm 5mm 5mm},clip, height=0.8\textwidth, width=\textwidth]{figures/mnist_iso.png}
		\end{subfigure}
	\caption{Comparison of various methods on a Deep Autoencoder on CURVES (top) and MNIST (bottom). Our mix approach converges faster than the baseline methods and uses relatively few ISO calls in comparison to $\acd$}
	\label{fig:mnist-curves_app}
\end{figure*}
		
{\bf Methods} The parameter selection for all the methods were carried as follows::
\begin{enumerate}[leftmargin=*, itemsep=0mm,partopsep=0pt,parsep=0pt]
\item \adam: We performed a grid search over $\alpha$ and $\varepsilon$ parameters of \adam so as to produce the best generalization on a held out test set. We found it to be $\alpha=10^{-3},\varepsilon=10^{-3}$ for CURVES and $\alpha=10^{-2},\varepsilon=10^{-1}$ for MNIST.
\item \acd: The regularization parameter $M$ was chosen as the largest value  such function value does not jump in first 10 epochs. We found it to be $M=10^3$ for both CURVES and MNIST. The sub-problem was solved with gradient descent \cite{Yair16} with the step-size of solver to be $10^{-3}$ and run till the gradient norm of the sub-problem is reduced below 0.1.
\end{enumerate}

\end{document}